\documentclass[preprint,12pt,number]{elsarticle}

\usepackage{amssymb}
\usepackage{amsmath, amssymb, bm}
\usepackage{booktabs}  
\usepackage{array, hyperref, url}
\usepackage{booktabs}
\usepackage{listings}
\usepackage{amsthm}
\newtheorem{theorem}{Theorem}
\newtheorem{lemma}[theorem]{Lemma}
\newtheorem{definition}{Definition}
\usepackage{amsmath}
\usepackage{algorithm,algorithmic}

\begin{document}

\newcommand{\R}{\mathbb{R}}
\newcommand{\Rd}{\R^{d}}
\newcommand{\Rn}{\R^{n}}
\newcommand{\Rm}{\R^{m}}
\newcommand{\RN}{\R^{N}}
\newcommand{\Rnn}{\R^{n\times n}}
\newcommand{\N}{\mathbb{N}}
\newcommand{\Z}{\mathbb{Z}}
\newcommand{\C}{\mathbb{C}}
\newcommand{\X}{\mathbb{X}}
\newcommand{\sigmatwo}{\sigma^{2}}
\newcommand{\phalf}{\frac{p}{2}}
\newcommand{\half}{\frac{1}{2}}
\newcommand{\inverse}{^{-1}}
\newcommand{\symmetricPart}{\text{sym}}
\newcommand{\antisymmetricPart}{\text{antisym}}
\newcommand{\abs}[1]{\left\lvert {#1} \right\rvert}
\newcommand{\id}[1][ ]{\mathtt{id}_{#1}}

\newcommand{\transpose}{^{\mathsf{T}}}
\newcommand{\squared}{^{2}}

\newcommand{\hadamard}{\odot}

\newcommand{\subscriptij}{_{i,j}}
\newcommand{\tzero}{t_{0}}
\newcommand{\eminusrt}{e^{-rt}}
\newcommand{\intzerot}{\int_{0}^{t}}

\newcommand{\suchThat}{\text{ s.t. }}
\newcommand{\argmin}{\mathrm{argmin}}
\newcommand{\argmax}{\mathrm{argmax}}
\newcommand{\partition}{\pi}
\newcommand{\timeHorizon}{T}
\newcommand{\timeWindow}{{[0,\timeHorizon]}}
\newcommand{\intZeroTimeHorizon}{\int_{0}^{\timeHorizon}}

\newcommand{\hurstExponent}{H}
\newcommand{\volterraKernel}{{K}_{\hurstExponent}}
\newcommand{\volterraProcess}{\zeta}
\newcommand{\covarianceKernel}{{R}_{\hurstExponent}}

\newcommand{\zeroSleqTtimeHorizon}{0\leq s\leq t \leq \timeHorizon}
\newcommand{\zerostTimeHorizon}{0\leq s, t \leq \timeHorizon}
\newcommand{\simplex}{\lbrace (s,t) \in \R\squared:\,  \zeroSleqTtimeHorizon\rbrace}
\newcommand{\timesSleqUleqT}{s \leq u \leq t}

\newcommand{\convexHull}{\mathrm{Conv}}

\newcommand{\restrictedto}[1]{\arrowvert_{#1}}

\newcommand{\trace}{\mathtt{trace}}

\newcommand{\positivePartOfMinimum}{\text{min}_{+}}

\newcommand{\dotEta}{\dot{\eta}}
\newcommand{\Kappa}{\mathfrak{K}}

\newcommand{\ones}{\mathbf{1}}

\newcommand{\infinity}{\infty}

\newcommand{\derivative}{^{\prime}}
\newcommand{\gradient}{\nabla}
\newcommand{\Fprime}{F\derivative}
\newcommand{\partialij}{\partial^{2}_{i,j}}
\newcommand{\gradx}{\nabla_{x}}
\newcommand{\gradz}{\nabla_{z}}
\newcommand{\Hessianx}{\nabla^{2}_{xx}}
\newcommand{\Hessianz}{\nabla^{2}_{zz}}

\newcommand{\norm}[1][\cdot]{\left\lVert {#1}\right\rVert}
\newcommand{\tripleNorm}[1]{\vert \vert \vert {#1}\vert \vert \vert }
\newcommand{\supNorm}[1][\cdot]{\lVert #1 \rVert_{\infty}}
\newcommand{\HoelNorm}[2][\cdot]{\lVert {#1} \rVert_{{#2}\text{-H\"ol}}}
\newcommand{\pvarNorm}[2][\cdot]{\lVert {#1} \rVert_{{#2}\text{-var}}}
\newcommand{\pvarNormInterval}[3][\cdot]{\lVert {#1} \rVert_{{#2}\text{-var}, {#3}}}
\newcommand{\lonenorm}[1][\cdot]{\norm[#1]_{1}}
\newcommand{\ltwonorm}[1][\cdot]{\norm[#1]_{2}}

\newcommand{\oneforms}[2]{ {\Omega}^{{1}} ({#1},{#2})}
\newcommand{\alphatilde}{\tilde{\alpha}}
\newcommand{\sectionsTM}[1][TM]{\Gamma(#1)}
\newcommand{\TxM}{T_x M}
\newcommand{\TmM}{T_m M}
\newcommand{\Homomorphisms}{\text{Hom}}
\newcommand{\manifold}{\mathcal{M}}

\newcommand{\banachSpace}{\mathcal{B}}
\newcommand{\pairing}[2]{\langle{#1},\, {#2} \rangle }
\newcommand{\Lone}{L^{1}}
\newcommand{\Ltwo}{L\squared}
\newcommand{\tripleNormClosure}{L^1(\timeWindow;\Ltwo(\Prob))}
\newcommand{\vectorSpace}{\boldsymbol{V}}
\newcommand{\smoothFunctions}[2]{C^{\infty}(#1,#2)}
\newcommand{\smoothCompactlySupportedFunctions}[2]{C^{\infty}_{c}(#1,#2)}
\newcommand{\Cpvar}[1][p]{C^{{#1}\text{-var}}}
\newcommand{\Contpvar}[3][p]{C^{{#1}\text{-var}}({#2},{#3})}
\newcommand{\diracDelta}{\delta}
\newcommand{\HoelderPaths}[1][\alpha]{C^{{#1}-\text{H\"ol}}}
\newcommand{\approxAdditivepVariation}{AA_{p\text{-var}}}
\newcommand{\semigroupP}{\mathtt{P}}
\newcommand{\semigroupT}{\mathtt{T}}
\newcommand{\generatorL}{{\mathsf{L}}}
\newcommand{\generatorA}{\mathsf{A}}
\newcommand{\CalphaHoelderLoc}{C^{\alpha\text{-H\"ol}}_{\text{loc}}}
\newcommand{\ContFunctionsOfEllipticPDEregularity}{\mathcal{C}^{\alpha}}
\newcommand{\sobolevSpace}[1][1,2]{W^{#1}}
\newcommand{\sobolevSpaceOneTwo}{\sobolevSpace}
\newcommand{\sobolevSpaceCompactSupport}[1][1,2]{\sobolevSpace[#1]_{0}}
\newcommand{\sobolevSpaceOneTwoCompactSupport}{\sobolevSpaceOneTwo_{0}}
\newcommand{\SobolevSpace}[1][1,2]{W^{#1}}

\newcommand{\Prob}{{{P}}}
\newcommand{\probabilityLaw}{\text{Law}}
\newcommand{\probabilityQ}{{{Q}}}
\newcommand{\probabilityDensity}{{p}}
\newcommand{\Expectation}{{{E}}}
\newcommand{\Variance}{\mathrm{Var}}
\newcommand{\CoVariance}{\mathrm{Cov}}
\newcommand{\correlation}{\mathrm{corr}}
\newcommand{\likelihood}{\mathcal{L}}
\newcommand{\loglikelihood}{\log \likelihood}
\newcommand{\sigmaAlgebra}{\mathfrak{F}}
\newcommand{\sigmaAlgebraG}{\mathcal{G}}
\newcommand{\setAlgebraA}{\mathcal{A}}
\newcommand{\setAlgebraB}{\mathcal{B}}
\newcommand{\orthogonal}{^{\perp}}
\newcommand{\measurableSpace}{\big(\Omega,\sigmaAlgebra \big)}
\newcommand{\probabilitySpace}{\big(\Omega,\sigmaAlgebra, \Prob \big)}
\newcommand{\filteredMeasurableSpace}{\big(\Omega,\sigmaAlgebra, (\sigmaAlgebra_t)_t \big)}
\newcommand{\stochasticBase}{\big(\Omega,\sigmaAlgebra = (\sigmaAlgebra_t)_t , \Prob  \big)}
\newcommand{\filtrationF}{\mathfrak{F}}
\newcommand{\internalHistory}{\filtrationF}
\newcommand{\iid}{\overset{\text{i.i.d.}}{\sim}}
\newcommand{\gaussian}[2]{\mathcal{N}({#1},{#2})}
\newcommand{\normalPDF}[3]{p_{\gaussian{#2}{#3}}\left( #1 \right)}
\newcommand{\WienerMeasure}[1][ ]{\mu_{#1}}
\newcommand{\brownianMotion}{W}
\newcommand{\geometricBrownianMotion}{X}
\newcommand{\compensator}{\Lambda}
\newcommand{\EDFfun}[1][e]{\hat{F}_{#1}}
\newcommand{\dirichletLaw}[1][\dirparam]{\Prob_{#1}}
\newcommand{\dirichletDensity}[1][\dirparam]{f_{#1}}
\newcommand{\targetDensity}{f}
\newcommand{\proposalDensity}{g}
\newcommand{\dirichlet}[1][\dirparam]{\text{Dir}_{#1}}
\newcommand{\boundedVariationPart}{A}
\newcommand{\martingale}{M}
\newcommand{\semimartingale}{S}
\newcommand{\mi}{I}

\def\one{\mbox{1\hspace{-4.25pt}\fontsize{12}{14.4}\selectfont\textrm{1}}}

\newcommand{\spaceX}{\mathcal{X}}
\newcommand{\spaceY}{\mathcal{Y}}

\newcommand{\jointeval}[1][\cdot]{\mu\left(#1\right)}
\newcommand{\prodeval}[1][\cdot]{\nu\left(#1\right)}
\newcommand{\donskervaradhanloss}{{v}}
\newcommand{\lossfun}{{\ell}}
\newcommand{\projection}{{p}}

\begin{frontmatter}


\author[label1,label2]{Taurai Muvunza}
\ead{t.muvunza@qmul.ac.uk}
\author{Egor Kraev}
\ead{egor.kraev@gmail.com}
\author[label2]{Pere Planell-Morell}
\ead{pere.planell@wise.com}
\author[label1,label3]{Alexander Y. Shestopaloff \corref{cor1}}
\ead{a.shestopaloff@qmul.ac.uk}

\cortext[cor1]{Corresponding author}

\title{MINERVA: Mutual Information Neural Estimation for Supervised Feature Selection}
\affiliation[label1]{organization={Queen Mary University of London},
             addressline={School of Mathematical Sciences},
             city={London},
             postcode={E1 4NS},
             country={United Kingdom}}
 \affiliation[label2]{organization={Wise Payments Ltd},
             addressline={65 Clifton Street},
             city={London},
             postcode={EC2A 4JE},
             country={United Kingdom}}
             
\affiliation[label3]{organization={Memorial University of Newfoundland},
             addressline={Department of Mathematics and Statistics},
             city={St. John's},
             postcode={A1C 5S7},
             country={Canada}}



\begin{abstract}
Existing feature filters rely on statistical pair-wise dependence metrics to model feature-target relationships, but this approach may fail when the target depends on higher-order feature interactions rather than individual contributions. We introduce Mutual Information Neural Estimation Regularized Vetting Algorithm (MINERVA), a novel approach to supervised feature selection based on neural estimation of mutual information between features and targets. We paramaterize the approximation of mutual information with neural networks and perform feature selection using a carefully designed loss function augmented with sparsity-inducing regularizers. Our method is implemented in a two-stage process to decouple representation learning from feature selection, ensuring better generalization and a more accurate expression of feature importance. We present examples of ubiquitous dependency structures that are rarely captured in literature and show that our proposed method effectively captures these complex feature-target relationships by evaluating feature subsets as an ensemble. 
Experimental results on synthetic and real-life fraud datasets demonstrate the efficacy of our method and its ability to perform exact solutions.

\end{abstract}




\begin{keyword}
feature selection \sep neural networks \sep mutual information \sep fraud



\end{keyword}

\end{frontmatter}



\section{Introduction}
\label{intro}
High dimensional data generally contains irrelevant and redundant features, which require large storage, high computation and lead to low performance models \citep{liu2022improving}. Effectively selecting important features in high-dimensional datasets is a long-standing challenge in machine learning and statistics \citep{koyama2022effective}. Methods of dimensionality reduction can be divided into two classes: feature selection and feature extraction. The goal of feature selection is to represent high-dimensional datasets with a subset of the original features. On the contrary, feature extraction methods such as the principal component analysis \citep{jolliffe2005principal} transforms the original features into new features by projecting the data as a linear combination of its original features that is represented only by the first few components. Since feature extraction methods preserve as much data variability as possible, their main drawback is the loss of physical meaning of the features \citep{nguyen2014effective,chandrashekar2014survey}. By selecting a subset of the original features, feature selection preserves the feature interpretability, making it a preferred choice in several domains \citep{tripathi2020interpretable,kim2015mind}.

Feature selection methods can be divided into two classes:  wrappers and filters.
The idea behind wrapper methods is that they consider the predictor algorithm as a black box and the predictor performance as the objective function to evaluate the feature subset. The combination of a subset of features that maximize the objective function is found through a subroutine heuristic search, with different features removed from the data \citep{kohavi1997wrappers}. Since the evaluation of subsets becomes an NP hard problem, wrappers are usually computationally expensive \citep{chandrashekar2014survey}. 

Filters utilize feature ranking techniques such as a score of dependence between features and target, and select a subset of features based on this score. Unlike wrappers that are model dependent, filters are applied before classification to remove less relevant features. The main challenge in filters is that of finding a subset of original features from a high dimensional dataset, such that a predictor algorithm that is trained on data only containing these features generates a classifier with the highest possible accuracy. From a theoretic standpoint, this makes discriminating between relevant and irrelevant features a ubiquitous problem \cite{kohavi1997wrappers,john1994irrelevant}.

Feature selection has been thoroughly investigated in literature, albeit under assumptions that do not apply to most real world scenarios. For example,  earlier studies in statistics by \cite{narendra1977branch}, \cite{draper1998applied}, \cite{miller1984selection}, \cite{devijver1982pattern}, and \cite{ben198235} addressed feature selection, but with a significant concentration on linear regression. The simplest feature selection method involves considering the $L_1$ regularization into the model, such as the Least Absolute Shrinkage and Selection Operator (Lasso) \citep{tibshirani1996regression}. Although Lasso-based feature selection methods are computationally inexpensive and widely applied, they are limited to linear models and may not be best used to describe nonlinear relationships.

Several methods to capture nonlinear relationships in feature selection have been proposed. The most common method involves assigning a statistical score to evaluate pairwise nonlinear relationship between the feature and target, and selecting the top features with the most relevance to the output \citep{peng2005feature}. Generally, widely used approaches include mutual information \citep{cover1999elements}, Hilbert-Schmidt Independence Criterion (HSCI) \citep{gretton2005measuring} and distance correlation \citep{szekely2009brownian}. Earlier versions of Lasso were limited to linear models, but later versions were developed that incorporate nonlinearity \citep{roth2004generalized,li2005lasso,yamada2014high}. 

Mutual information (MI) is defined as a measure of the amount of information one random variable contains about another \citep{cover1999elements}. MI has been widely applied in data science as a fundamental quantity for measuring the relationship between random variables \citep{belghazi2018mutual}. Over the years, MI-based feature selection methods have gained popularity due to their effectiveness, ease of use and strong theoretical foundations rooted in information theory. More precisely, MI is used in feature selection to find the minimal feature subset with maximum MI with respect to the target variable \cite{brown2012conditional,liu2022improving}. 

Since searching for the optimal feature subset is computationally intractable, numerous MI-based feature selection methods employ Maximum Relevance with Minimum Redundancy (MRMR) \cite{peng2005feature,nguyen2014effective,meyer2008information,yang1999data}, a technique that has demonstrated competitive performance in dimensionality reduction \citep{zebari2020comprehensive}. 

Despite being a pivotal quantifier in feature selection, MI has historically been difficult to compute \citep{paninski2003estimation}. In addition, exact estimation of MI is tractable for discrete random variables, and in selective cases where the closed form probability density function of the random variables is known \citep{belghazi2018mutual}. Existing approaches rely on nonparametric differential entropy to estimate the MI of continuous random variables \citep{beirlant1997nonparametric} while others first estimate the density using kernel density estimators such as $k$-nearest neighbors \citep{kraskov2004estimating,wang2006nearest,leonenko2008class}. However, nonparametric methods are inefficient \citep{beirlant1997nonparametric} and kernel density estimators generally fail to converge to the true measure \citep{perez2008estimation}.  

A new approach to estimate MI based on neural networks was proposed in \cite{belghazi2018mutual}. The authors propose Mutual Information Neural Estimation (MINE), a neural estimator for MI based on the dual representation of Kullback-Leibler (KL)-divergence \cite{kullback1997information}. A more general expression of MI is given by:
\begin{equation}
I(X; Y) = \int_{x \times y} \log \frac{d\mathbb{P}_{XY}}{d\mathbb{P}_X \otimes \mathbb{P}_Y} \, d\mathbb{P}_{XY}
\end{equation}
\noindent where $\mathbb{P}_{XY}$ is the joint probability distribution; and $\mathbb{P}_{X}$ and $\mathbb{P}_{Y}$ are the marginals. 
In this paper, we introduce a novel approach to supervised feature selection based on neural estimation of mutual information between features and targets. We utilize MINE since it is a consistent, flexible and scalable method for estimating mutual information \citep{belghazi2018mutual}. We propose a two-stage framework that combines Mutual Information Neural Estimation with Regularized Vetting to learn complex dependency structure between random variables. Our approach relies on a carefully designed loss function to simultaneously estimate mutual information and perform feature selection by integrating a variational mutual information estimator with sparsity-inducing regularizers. Through experiments on challenging synthetic and real-world feature selection problems, we show that the proposed method compare favorably with existing feature selection methods. 
MINERVA belongs to the class of filters, and utilizes the mutual information as score.

The remainder of the paper is organized as follows: Section \ref{sec.mine} presents the background on approximating mutual information using neural networks, which serves as the foundation of our method. We introduce our methodology in Section \ref{sec.method} and discuss experiments in Section \ref{sec.experiments}. Finally, we conclude in Section \ref{sec.conclusion}.

\section{Background: Neural estimation of mutual information}
\label{sec.mine}

Consider two random variables $X$ and $Y$, their mutual information can be defined as the reduction in uncertainty of $X$ given the knowledge of $Y$ \citep{cover1999elements}:
\begin{equation}
    I(X;Y) = H(X)-H(X|Y)
\end{equation}
\noindent where $H$ is the Shannon entropy \citep{shannon1948mathematical} and $H(X|Y)$ is the conditional entropy of $X$ given $Y$. MI can also be expressed as KL-divergence between the joint distribution and the product of the marginals: $I(X;Y)=D_{KL}(P(X, Y) \parallel P(X) \otimes P(Y))$, where $D_{KL}$ is the Kullback–Leibler divergence, and $\displaystyle P_{X}\otimes P_{Y}$ is the outer product distribution which assigns probability $P_{X}(x)\cdot P_{Y}(y)$ to each $x,y$ pairs.

MINE uses the Donsker-Varadhan (DV) dual representation of KL-divergence \citep{donsker1983asymptotic}:
\begin{align}
    D_{K L}(\mathbb{P} \| \mathbb{Q})=\sup _{T: \Omega \rightarrow \mathbb{R}} \mathbb{E}_{\mathbb{P}}[T]-\log (\mathbb{E}_{\mathbb{Q}}[e^{T}]),
\end{align}
\noindent where $\mathbb{P}$ and $\mathbb{Q}$ are arbitrary distributions and the supremum is taken over all functions $T$ such that the two expectations are finite. $T$ is an arbitrary function mapping from the sample space to real number $\mathbb{R}$ and $\mathcal{F}$ denotes any class of integrable functions $T: \Omega \rightarrow\mathbb{R}$. DV dual representation allows for the estimation of a variational bound of the KL divergence and finding the tightest point of this bound, specified by:
\begin{align}
    D_{KL}(\mathbb{P} \| \mathbb{Q}) \geq \sup _{T \in \mathcal{F}} \mathbb{E}_{\mathbb{P}}[T]-\log (\mathbb{E}_{\mathbb{Q}}[e^{T}])
\end{align}
Given the above expression, mutual information can be rewritten in terms of KL divergence between the joint and product of marginals:
\begin{align}
    I(X ; Y) \geq I_{\Theta}(X, Y)=\sup _{\theta \in \Theta} \mathbb{E}_{\mathbb{P}_{X Y}}[T_{\theta}]-\log (\mathbb{E}_{\mathbb{P}_{X} \otimes \mathbb{P}_{Y}}[e^{T_{\theta}}])
\end{align}
Finally, since true distributions are unknown, we resort to an empirical estimator which replaces the expectations with sample-based approximations:
\begin{align}
    I \widehat{(X ; Z)}_{n}=\sup _{\theta \in \Theta} \mathbb{E}_{\mathbb{P}_{X Z}^{(n)}}[T_{\theta}]-\log (\mathbb{E}_{\mathbb{P}_{X}^{(n)} \otimes \hat{\mathbb{P}}_{Z}^{(n)}}[e^{T_{\theta}}])
\end{align}
\noindent where $\mathcal{F}$ is chosen to be the family of functions $T_{\theta}: \mathcal{X} \times \mathcal{Y} \rightarrow \mathbb{R}$ parameterized by a deep neural network with parameters $\theta \in \Theta$. The main advantage of representing mutual information as dual representation of KL-divergence is that the estimator no longer depends on intractable probabilities to estimate the expectation in the bound, as samples of $X$ and $Y$ can be directly used instead.

Given samples  	$(x_1, y_1), 	\dots,	(x_n, y_n)$, from the joint distribution of $X$ and $Y$, the representation in equation (5) can be used to estimate the mutual information of the two random variables, where the functions $T_\theta$ are parameterized by a neural network and the empirical objective function is maximized  by gradient ascend in the parameter space $\Theta$. Since empirical samples are used in DV dual representation as shown in equation (5), the first expectation $\mathbb{E}_{\mathbb{P}_{XY}} \mathcal{F}(X,Y)$ is computed by:
\begin{equation*}
	\frac{1}{n}\sum_{i=1}^{n} \mathcal{F}(x_i, y_i)
\end{equation*}
\noindent and the second expectation $\mathbb{E}_{\mathbb{P}_{X} \otimes \mathbb{P}_{Y}}\mathcal{F}(X,Y)$ is given by:
\begin{equation*}
	\frac{1}{n}\sum_{i=1}^{n} \exp\left( \mathcal{F}(x_i, y_{\sigma(i)})\right)
\end{equation*}
\noindent where $\sigma$ is a permutation used to shuffle the $Y$ samples and transform the samples $(x_1, y_1), 	\dots,	(x_n, y_n)$ into samples from $\Prob_{X} \otimes \Prob_{Y}$. We rely on this approach to construct a feature selection filter based on neural estimation of mutual information.

\section{Methodology}
\label{sec.method}
\begin{definition}
Let $\spaceX \subset \Rd$ and $\spaceY \subset \R^{e}$ represent sample spaces where
$X$ and $Y$ are random variables taking values in $\spaceX$ and $\spaceY$
respectively. We define $Y$ as the target of a prediction task, and $X$ as a set of features to use in the prediction task.
\end{definition}

Given $n$ empirical samples $(x_1, y_1), \dots, (x_n, y_n)$ from the joint distribution $\Prob_{X Y}$, a permutation $\sigma \in S_n$ where $S_n$ is a set of all possible permutations of indices $\{1,...,n\}$, a real valued function $f: \spaceX \times \spaceY \rightarrow \R$, and a $d$-dimensional vector $\projection \in \Rd$, we estimate the expectations in equation (5) as follows:
\begin{equation}
	\begin{split}
		\jointeval[f, \projection] &= \frac{1}{n} \sum_{i=1}^{n} f(\projection \hadamard x_i, y_i),
		\\
		\prodeval[f, \projection] &= \frac{1}{n} \sum_{i=1}^{n} \exp\left(f(\projection \hadamard x_{\sigma(i)}, y_i)\right),
	\end{split}
\end{equation}
\noindent where $\projection \hadamard x_i$ is the Hadamard product of 
$\projection$ and $x_i$, and $\projection$ denotes the weights of the feature vector. The first term $\jointeval[f, \projection]$ is used to approximate  $\mathbb{E}_{\mathbb{P}_{XY}} \mathcal{F}(X,Y)$,
while $\prodeval[f, \projection]$ approximates the second term $\mathbb{E}_{\mathbb{P}_{X} \otimes \mathbb{P}_{Y}}\mathcal{F}(X,Y)$ in equation (5).

\begin{definition}
    Let $f_\theta$, $\theta \in \Theta$ be a family of measurable functions $f_\theta: \spaceX \times \spaceY \rightarrow \R$ parameterized by the parameter $\theta \in \Theta$ of the neural network, we define an approximation of the negative of mutual information of $p \odot X$ and $Y$, denoted by $\donskervaradhanloss(\theta, \projection)$ as follows:
    \begin{equation}
	\donskervaradhanloss(\theta, \projection) = - \jointeval[f_\theta, \projection] + \log \left(\prodeval[f_\theta, \projection]\right)
\end{equation}
\end{definition}

\subsection{Loss function}
We design the loss function to incorporate regularization in the model as it is an effective way to induce sparsity in feature selection methods \citep{koyama2022effective}.
\begin{definition}
    Let $c_1$, $c_2$, $a$ be non-negative real coefficients, we define the loss function as:
    \begin{equation}
	\lossfun(\theta, p, c_1, c_2, a)
 \,\,
	=  
 \,\,
	\donskervaradhanloss(\theta, \projection)
		 + 
		c_1 \lonenorm[\frac{\projection}{\ltwonorm[\projection]}]
		 + 
		c_2 \left( \ltwonorm[\projection] - a \right)\squared,
\end{equation}
\noindent where $\lonenorm$ denotes $\Lone$-norm and $\ltwonorm$ denotes $\Ltwo$-norm.
\end{definition}
The function $\lossfun$ is the loss function that should be minimized. 
It consists of three terms. The first term $\donskervaradhanloss(\theta, \projection)$
is the discretisation of the function that appears in the Donsker-Varadhan representation of the KL-divergence. It approximates the negative mutual information between the target $Y$, and the $\projection$-weighted features. 

The second term $\lonenorm[\frac{\projection}{\ltwonorm[\projection]}]$ is a regularization term on the weights $\projection \in \Rd$. The regularization term induces sparsity in the model by pushing the weights of non-relevant features to zero. Introducing the regularization on the scaled norm ensures penalization of the relative distribution of the weights without affecting its overall magnitude. In addition, the normalization keeps the size of $p$ constant, allowing the focus to be purely on sparsity. 
Finally, the third term $\left( \ltwonorm[\projection] - a \right)\squared$ controls the euclidean norm of the weights $\projection \in \Rd$ by penalizing the square of the difference between said norm and the target norm $a$. This is meant to prevent the weights of relevant features from diverging.

Our feature selection method involves identifying a minimizer $\hat{\theta}$ of 
\begin{equation*}
\theta \longmapsto \donskervaradhanloss(\theta, \ones),
\end{equation*}
\noindent where $\ones = (1, \dots, 1) \in \Rd$, and then using this $\hat{\theta}$ as the initialisation of the gradient descent for the minimisation of 
\begin{equation*}
	\theta, \projection \longmapsto \lossfun\left(\theta, \projection, c_1, c_2, \sqrt{d}\right).
\end{equation*}
\noindent 


We let $d$ be the average number of selected features and introduce $\sqrt{d}$ as a scaling factor of the drift term in the regularizer. This ensures stability of the regularization effect as the number of selected features changes. For instance, when the number of selected features is small, $\sqrt{d}$ will also be small and reduces the impact of the drift term from over-penalizing the small selection. After the gradient descent stops, we select the features that correspond to non-zero weights of $\projection$.

The architecture of the neural network used in the parametrization  of the test functions $f_\theta$ is illustrated in Figure \ref{fig:model}. We separate the input into categorical and float features. We use the embedding layer to represent categorical features into a lower-dimensional space. To ensure stable numerical values and prevent large gradients, we pass categorical features through a soft clamp operation. The projection layer transforms the joint samples into a space suitable for estimation of mutual information.  The residual blocks process the feature representations and stabilize the learning of complex interactions in the data. 

Details on the implementation of our approach is shown in Algorithm \ref{algo.minerva}. The code to reproduce results reported in this paper is available at the project's github \href{https://github.com/transferwise/minerva}{repository}. 
We train the feature selection method in a two-stage process. First, in steps (1-6) of Algorithm \ref{algo.minerva}, we fix $p=\textbf{1}$ and train MINERVA to explore the dependency between $X$ and $Y$ without any feature selection constraints. This is important for learning  mutual information in a stable optimization process. Second, the learned $\theta$ is introduced as initialization in the feature selection step (7-14) and the network parameter $\varphi$ is updated to $\theta$, but subsequently optimized with sparsity-inducing regularizers. The goal of the second stage is to fine-tune the learned mutual information estimator while introducing regularization to select important features. Decoupling the learning process improves generalization by preventing the regularizers from interfering with the network’s ability to learn the joint distribution $P_{XY}$.

\begin{algorithm}
	\caption{Mutual Information Neural Estimation Regularized Vetting}
	\label{algo.minerva}
	\begin{algorithmic}[1]
		\REQUIRE
		random variables
		$X\in \spaceX$,
		$Y\in \spaceY$,
		hyperparameters
		$r>0$, $c_1\geq 0$, $c_2\geq 0$.
		\STATE $\theta \leftarrow$ initialise network parameters
		\REPEAT
		\STATE Draw $n$ samples $(x_1, y_1), \dots, (x_n, y_n)$ from the joint distribution $\Prob_{XY}$
		\STATE Sample shuffling permutation $\sigma$ from $S_n$
		\STATE Update $\theta \leftarrow \theta - r\gradient_{\theta}\donskervaradhanloss(\theta, \ones)$
		\UNTIL{ convergence}
		\STATE Initialise $\varphi \leftarrow \theta$, $\projection \leftarrow \ones$.
		\REPEAT
		\STATE Draw $n$ samples $(x_1, y_1), \dots, (x_n, y_n)$ from the joint distribution $\Prob_{XY}$
		\STATE Sample shuffling permutation $\sigma$ from $S_n$
		\STATE Update $\varphi \leftarrow \varphi - r \gradient_{\varphi}\lossfun(\varphi, \projection, c_1, c_2, \sqrt{d})$
		\STATE Update $\projection \leftarrow \projection - r \gradient_{\projection}\lossfun(\varphi, \projection, c_1, c_2, \sqrt{d})$
		\UNTIL{ convergence}
		\RETURN $\left\lbrace i: \abs{\projection_i} > 0\right\rbrace$
	\end{algorithmic}
\end{algorithm}

\begin{figure}[]
    \centering
    \includegraphics[width=1.05\linewidth]{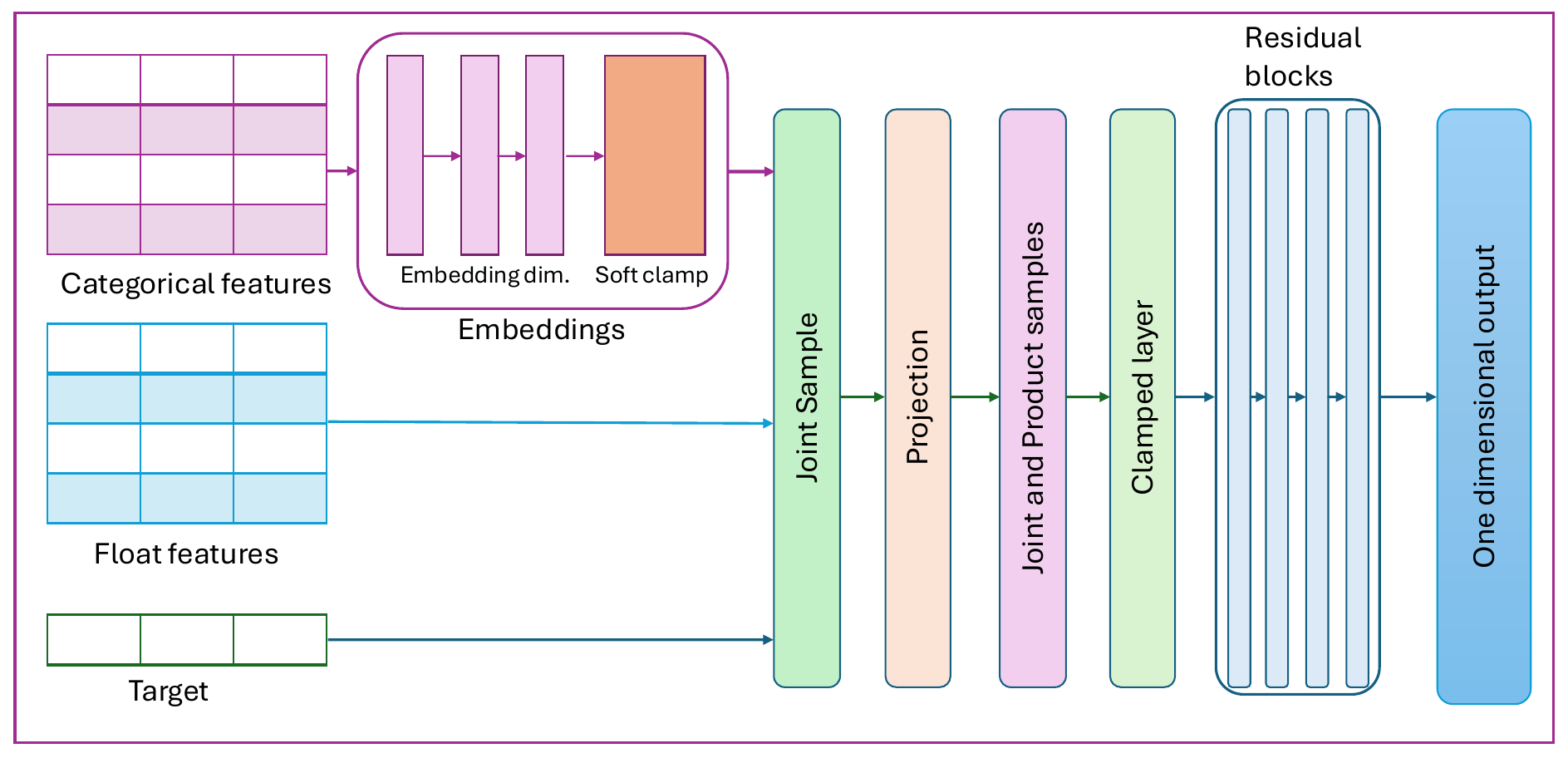}
     \vspace{-0.5cm}
    \caption{Neural network architecture for MINERVA}
    \label{fig:model}
\end{figure}

\section{Experiments}
\label{sec.experiments}
\subsection{Set-up}
For experiments using synthetic data, the regularization coefficient was set to 1 as larger values excessively penalized the weights. With regard to real-world data, the regularization coefficient was empirically selected and reported in the results.  The drift term $a$ was set to 1 in all experiments, while the learning rate was fixed to 0.0001. Since non-zero weighted features are selected after training, we set the threshold $\epsilon$ for $p$ to 0.00001 in all experiments. 

We conduct extensive experiments on synthetic and real-world datasets to validate our approach. When using synthetic data, the subset of features that generated the target is known, and we evaluate our method by reconciling the selected features against expected features.   

\begin{definition}
    Let $s$ be a subset of selected features $\lbrace1, \dots, d\rbrace$ and $t$ denote the set of features that generated the target $t\subset \lbrace 1, \dots, d\rbrace$ we define exact and non-exact selection such that $s$ is exact if $s=t$, and non-exact, otherwise. Moreover, if the selection is non-exact, either $t \not \subset s$ or $s \supsetneq t$. In the former case, the non-exact selection is classified as Type I, while in the latter case, it is classified as Type II. 
\end{definition}

Non-exact selections of Type I compromise the downstream prediction task 
because they subtract information that is relevant for the prediction. Non-exact selections of Type II might not reduce the dimensionality of the data, but they do not compromise downstream tasks. 

\subsection{Benchmark methods}
We test our feature selection method against five benchmark methods: KSG \citep{kraskov2004estimating}, Boruta \citep{kursa2010boruta}, HSCI Lasso \citep{yamada2014high}, Recursive Feature Elimination (RFE) \citep{guyon2002gene}, Feature Ordering by Conditional Independence (FOCI) \citep{azadkia2021simple} and Random Forest (RF) \citep{iranzad2024review}. KSG is a regression-based method for feature selection that uses approximation of mutual information proposed in \cite{kraskov2004estimating}. The method estimates mutual information $I(X_i; Y)$ for all $\{i = 1, \dots, d\}$ and selects features $k$ such that $I(X_k; Y) > \epsilon$ for a given threshold $\epsilon \geq 0$. The estimation of $I(X_i; Y)$ is performed using the KSG estimator which employs nonparametric techniques to estimate entropy based on $k$-nearest neighbor distances.
Boruta is a random forest-based wrapper method which utilizes an importance score to compare the significance of original features against randomised copies. 

HSIC Lasso is a kernel method that captures non-linear input-output dependence based on maximizing:
\begin{equation}
    - \sum_{k=1}^{d} \alpha_k \text{HSIC}(\bm{u}_k, \bm{y})
    + \frac{1}{2} \sum_{k,l=1}^{d} \alpha_k \alpha_l \text{HSIC}(\bm{u}_k, \bm{u}_l),
\end{equation}
\noindent where $\text{HSIC}(\bm{u}_k, \bm{y})$ is the kernel-based independence measure of the Hilbert-Schmidt independence criterion \citep{gretton2005measuring}, and $\bm{\alpha_k, \alpha_l}$ denote the coefficients of the features. 

RFE recursively selects features by training an estimator, ranking features by importance, and pruning the least important ones until the desired number of features is reached. This process is repeated on smaller feature subsets to optimize selection. FOCI is a nonlinear, non-parametric variable selection algorithm based on a new measure of conditional dependence of variables $Y$ and $\boldsymbol{Z}$ given $\boldsymbol{X}$, specified as 
\begin{equation}
    T = T(Y, \mathbf{Z}|\mathbf{X}) := \frac{\int \mathbb{E}(Var(\mathbb{P}(Y \geq t|\mathbf{Z}, \mathbf{X})|\mathbf{X}))d\mu(t)}{\int \mathbb{E}(Var(1_{\{Y \geq t\}}|\mathbf{X}))d\mu(t)}
\end{equation}
\noindent where $1_{\{Y \geq t \}}$ is the indicator function of the event $\{Y \geq t \}$. We used RF to rank features based on their importance scores, and select the top-ranked features while discarding the less important ones. In addition, the benchmarks were optimized using cross-validation to selected the most informative features.

\subsection{Synthetic Datasets}
We investigate the phenomenon in which a target variable $Y$ depends on whether
two independent discrete random variables $X_{k_0}$  and $X_{k_1}$ are equal. This kind of dependence is relevant in financial transactions. For example, if two independent transactions share the same device ID but originate from different users, this could indicate account takeover or fraud. Despite the relevance of this problem in practice, existing feature selection methods do not capture this dependence.

\subsubsection{Experiment A}
Let $d$ be a positive integer, and  $m$ be another positive integer larger than $2$. For $i \in \{1, \dots, d\}$, we let $X_i$ be a random positive integer taking values in $i \in \{1, \dots, m\}$. Random variables $X_1, \dots, X_d$ are assumed to be independent and distributed identically. We fix two integers $ 1 \leq k_0 < k_1 \leq d$
and define
\begin{equation}
	\label{eq.exp1Y}
	Y = \one\left\lbrace
	X_{k_0} = X_{k_1}
	\right
	\rbrace
	=
	\begin{cases}
		1 & \text{ if } X_{k_0} = X_{k_1}
		\\
		0 & \text{ otherwise}.
	\end{cases}
\end{equation}

We address the problem of predicting $Y$ from the feature vector $(X_1, \dots, X_d)$,
and aim to identify the subset of features that are most relevant for accurate prediction. 
\begin{lemma}
    Feature selection methods that rely on a dependence metric $h(X_i, Y)$ to evaluate individual features are inherently flawed as any  pair-wise assessment of $(X_i, Y)$ is bound to fail since $Y$ is independent of each individual feature $X_i$. Exact selection is only possible when considering the joint information provided by the entire feature set $X_1, \dots, X_d$.
\end{lemma}

\begin{proof}
    \label{lemma.experiment1}
	Let $m>2$ be a positive integer and $X_1, \dots, X_d$ be independent,      
        identically distributed variables with $\Prob(X_1 = n) = 1 / m$ for $n= \{1, \dots, m\}$. 
	Let $k_0$ and $k_1$ be two distinct positive integers less than or equal to 
        $d$, and $Y$ be as in equation \eqref{eq.exp1Y}.
	Then, for all $i= \{1, \dots, d\}, $ $I(X_i; Y) = 0,$
		\label{eq.exp1pairwisemi}	
        namely $X_i$ and $Y$ are independent. 
	Moreover,
	\begin{equation}
		\label{eq.exp1mi}
		I(X_{k_0}, X_{k_1}; Y) =	\frac{m-1}{m} \log\left(\frac{m}{m-1}\right)
			+ \frac{1}{m}\log m,
	\end{equation}
	\noindent and $I(X_{k_0}; Y \lvert X_{k_1}) = I(X_{k_1}; Y \lvert X_{k_0})  = I(X_{k_0}, 
        X_{k_1} ; Y)$. Full proof is provided in \ref{app1}
\end{proof}

Table \ref{tab:experimentA} summarizes the results of Experiment A. We generated 30 features and 50000 samples; and expected the models to select features 3 and 8. However, results show that HSCI Lasso, KSG, Boruta and RF could not compute the exact solution, as they selected all the features. Their selection are non-exact of Type II. This is not surprising since existing methods compare pair-wise dependencies instead of considering the entire feature set as an ensample. Moreover, \cite{belghazi2018mutual} showed that MINE exhibited a marked improvement over KSG \citep{kraskov2004estimating} when estimating mutual information. FOCI and MINERVA selected exact features, demonstrating significant gain over existing baselines. For this experiment, our results indicate that tree-based algorithms such as Boruta and Random Forest perform poorly while FOCI, which depends on nonlinear statistical dependence, and our proposed neural network-based method achieved exact selection. 

\begin{table}[h!]
    \centering
     \caption{Comparison of feature selection methods.}
    \renewcommand{\arraystretch}{1.2} 
    \setlength{\tabcolsep}{8pt}       
    \begin{tabular}{l l l l}
        \toprule
        \textbf{Method} & \textbf{Selected} & \textbf{Expected} & \textbf{Evaluation} \\
        \midrule
        KSG & 1, \dots, 30 & 3, 8 & Non-exact Type II \\
        HSIC Lasso & 1, \dots, 30 & 3, 8 & Non-exact Type II \\
        Boruta &1, \dots, 30 & 3, 8 & Non-exact Type II \\
        RFE  & 1, \dots, 30 & 3, 8 & Non-exact Type II \\
        Random Forest & 1, \dots, 30 & 3, 8 & Non-exact Type II \\
        FOCI  & \textbf{3, 8}  &   \textbf{3, 8}        & \textbf{Exact} \\

        \textbf{MINERVA} & \textbf{3, 8} & \textbf{3, 8} & \textbf{Exact} \\
        \bottomrule
    \end{tabular}
    \label{tab:experimentA}
\end{table}

\subsubsection{Experiment B}
We also evaluate MINERVA in the context of predicting the target $Y$, where $Y$ is a nonlinear function of continuous features that depend on whether two independent variables are equal, as defined in Experiment A. For instance, given two continuous nonlinear functions $f_1$ and $f_2$, the target variable is determined by a transformation of certain continuous features via $f_1$  when two discrete variables are equal, and via $f_2$ when they are not.

Suppose $d_1$ and $d_2$ are positive integers, and $\{X_1, \dots, X_{d_1}\}$ are i.i.d random variables such that $\Prob(X_1 = k) = 1/m$ for $k=1, \dots, m$, for some positive integer $m > 1$. Let $\{X_{d_1 + 1}, \dots, X_{d_1 + d_2}\}$ be independent, identically distributed random variables with uniform distribution on the unit interval. It suffices that $\{X_1, \dots, X_{d_1}\}$ and $\{X_{d_1 + 1}, \dots, X_{d_1 + d_2}\}$ are independent. Given $k_0, k_1$ to be distinct positive integers smaller than or equal to $m$,  and $n < d_2$, such that $d_1 < j_0 <  \dots < j_n \leq d_1 + d_2$ and $d_1 < i_0 <  \dots < i_n \leq d_1 + d_2$, we define
\begin{equation}
	\label{eq.exp2target}
	Y 
	=
	\begin{cases}
		\sum_{\ell = 1}^{\ell=n} \alpha_\ell \sin\left(2\pi X_{j_\ell}\right)  
		& \text{ if } X_{k_0} = X_{k_1}
		\\
		\sum_{\ell = 1}^{\ell=n} \beta_\ell \cos\left(2\pi X_{i_\ell}\right)  
		& \text{ otherwise}.
	\end{cases}
\end{equation}
\noindent where $\alpha_l$ and $\beta_l$ are coefficients of the sine and cosine terms. In this setting, $Y$ depends on a nonlinear function if $X_{k0}=X_{k1}$, and another if $X_{k0}\neq X_{k1}$. We address the task of predicting $Y$ given a continuous feature vector $(X_1, \dots, X_{d_1}, X_{d_1 + 1}, \dots, X_{d_1 + d_2})$, and select the features that are most relevant for the prediction.

We set the number of features to 40 and generated 50000 samples. We expected 10 features to be selected: $\textbf{[6, 8, 14, 18, 19, 20, 23, 24, 28, 31]}$. Table \ref{tab:experiment1B} summarizes our main findings. The number of selected features was fixed at 10 to ensure a standardized comparison across models. KSG selected features of the non-exact Type 1 of which 7 were among the expected features. Of the 10 features selected by HSCI Lasso and Boruta, 8 were among the expected features, which is a slight improvement over KSG. Half of the features selected by RFE and RF were among the expected features. Meanwhile, FOCI performed poorly with only 3 expected features. However, none of the baselines was able to perform an exact feature selection. MINERVA showed a marked improvement over the baselines by performing an exact selection. This is largely due to the model's ability to consider joint information of the feature set as an ensample.

\begin{table}[h]
    \centering
        \caption{Experiment B: Performance evaluation (NE = Non-Exact).}
    \renewcommand{\arraystretch}{1.2} 
    \setlength{\tabcolsep}{8pt}       
    \begin{tabular}{l l l l}
        \toprule
        \textbf{Method} & \textbf{Selected} & \textbf{Evaluation} \\
        \midrule
        KSG &  
        \small{\textbf{14}, \textbf{18}, \textbf{19}, \textbf{20}, \textbf{23}, 25, \textbf{28}, \textbf{31}, 34, 38}  &  
        NE Type I \\
        HSIC Lasso &  
        \small{4, 11, \textbf{14}, \textbf{18}, \textbf{19}, \textbf{20}, \textbf{23}, \textbf{24}, \textbf{28}, \textbf{31}} &   
        NE Type I \\
        Boruta &  
        \small{\textbf{14}, \textbf{18}, \textbf{19}, \textbf{20}, \textbf{23}, \textbf{24}, \textbf{28}, \textbf{31}, 37, 38} &    
        NE Type I \\
        RFE & 
        \small {9, 12, \textbf{14}, 15, 16, \textbf{20}, \textbf{24}, \textbf{28}, \textbf{31}, 37} &
        NE Type I \\
        Random Forest &
        \small{11, 12, \textbf{14}, \textbf{19}, \textbf{23}, \textbf{28}, \textbf{31}, 37, 38, 40} &
        NE Type I \\
        FOCI  &
        \small {1, 7, \textbf{8}, 10, 11, \textbf{19}, 25, 26, \textbf{31}, 33} &
        NE Type I \\
        \textbf{MINERVA} & \textbf{ 
        \small{6, 8, 14, 18, 19, 20, 23, 24, 28, 31}} & 
       \textbf{
        Exact} \\
        \bottomrule
    \end{tabular}
    \label{tab:experiment1B}
\end{table}

To validate our method, we employ gradient boosting method to evaluate the quality of selected features by assessing their contribution to predictive performance. Thus, given a set of features selected by different feature selection methods, we train a gradient boosting model using each feature subset and compare its predictive accuracy. This allows us to quantify how well each selection method captures the most informative features for predicting the target variable $Y$. 
We evaluate the predictive performance of gradient boosting method using both in-sample $R^2$ (which measures how well the model fits the training data) and out-of-sample $R^2$ (which evaluates generalization performance on unseen data). A higher out-of-sample $R^2$ indicates that the selected features contribute effectively to prediction while avoiding overfitting. We split the data into 80\% for training and 20\% for out-of-sample testing.

The results of gradient boosting method are shown in Table \ref{tab:experiment1Baccuracy}. When all the 40 features are used, out-of-sample $R^2$ is 79.90\%, which is higher than all the baselines. Despite selecting an identical subset of expected features, Boruta achieves marginally higher out-of-sample performance (70.0\%) compared to KSG (69.8\%). While RFE and RF achieve comparable results (62.6\% and 62.13\%, respectively), FOCI underperforms (62.13\%) - an expected outcome given its selection of the fewest correct features. Overall, MINERVA achieves the best performance for both in-sample and out-of-sample $R^2$, reaching 84.69\%. The results validate the superiority of MINERVA in selecting the most informative features that are crucial for predicting the target variable.

\begin{table}[h]
    \centering
    \caption{Experiment B - Accuracy of a gradient boosting model}
    \renewcommand{\arraystretch}{1.2} 
    \setlength{\tabcolsep}{10pt}       
    \begin{tabular}{l c r r}
        \toprule
        \textbf{Method} & \textbf{\# of Features} & \textbf{In-Sample \(R^2\)} & \textbf{Out-of-Sample \(R^2\)} \\
        \midrule
        All Features & 40 & 0.8615 & 0.7990 \\
        KSG & 10 & 0.7647 & 0.6980 \\
        HSIC Lasso & 10 & 0.7717 & 0.7004 \\
        Boruta & 10 & 0.7669 & 0.7023 \\
        RFE  & 10 &  0.7000   & 0.6260 \\
        Random Forest  & 10  & 0.7010 & 0.6213 \\
        FOCI &10  & 0.6528  & 0.5857 \\
        \textbf{MINERVA} & \textbf{10} & \textbf{0.8799} & \textbf{0.8469} \\
        \bottomrule
    \end{tabular}
    \label{tab:experiment1Baccuracy}
\end{table}

\subsection{Real-life Dataset}
We investigate the performance of MINERVA on a real-world fraud dataset from a financial company. The dataset consists of 3 million samples and 214 processed features. The aim is to determine a subset of features that are more informative for predicting fraud. However, the dataset is highly imbalanced, with only 0.1\% frequency of positive labels. In financial risk management, the cost of misclassifying a fraudulent transaction as normal is often much higher than the cost of the reverse error. Therefore, reducing noise and eliminating redundant features is essential for enhancing predictive performance.

Previous studies have addressed the issue of imbalanced data by penalizing wrong classification of training samples \citep{domingos1999metacost}, and either under-sampling the majority class or oversampling the minority class \citep{cheng2020graph}. We apply Synthetic Minority Over-sampling Technique (SMOTE) \citep{chawla2002smote} to handle data imbalance. SMOTE creates clusters around each minority observation by generating minority samples that are within the neighbourhood of the observed samples. 

We separated data into 3 equal sets, and performed over-sampling with SMOTE on the first and second data sets. The first set was used to train the feature extractor. The selected features were evaluated using Fast and Lightweight Auto-Machine Learning library (FLAML) \citep{wang2021flaml}. Given the size of our data set, we employ FLAML since it exploits the structure of the search space to determine a search order optimized for both cost and error in finding accurate models. We used the second set to train FLAML while the last set was left imbalanced and used for testing. Benchmark models were optimized to select the most important features based on their objective function. Table \ref{tab:feature_selection} summarizes our findings. 

\begin{table}[h]
\centering
\caption{Performance comparison of feature selection methods.}
\renewcommand{\arraystretch}{1.1} 
\setlength{\tabcolsep}{4pt} 
\begin{tabular}{l@{\hskip 4pt}c@{\hskip 4pt}c@{\hskip 4pt}c@{\hskip 4pt}c@{\hskip 4pt}c@{\hskip 4pt}c}
\hline
 & All & Minerva$10^3$ & Minerva$10^4$ & HSCI & Boruta & KSG \\
\hline
\# of Features & 214 & 160 & 90 & 188 & 36 & 197 \\
In-sample Recall & 1.000 & 1.000 & 1.000 & 1.000 & 0.999 & 1.000 \\
Out-Sample Recall & \textbf{0.573} & \textbf{0.573} & 0.570 & \textbf{0.573} & 0.531 & 0.573 \\
In-sample Precision & 1.000 & 1.000 & 1.000 & 1.000 & 1.000 & 1.000 \\
Out-sample Precision & 0.935 & 0.933 & 0.915 & 0.928 & 0.861 & \textbf{0.937} \\
In-sample PR-AUC & 1.000 & 1.000 & 1.000 & 1.000 & 1.000 & 1.000 \\
Out-sample PR-AUC & \textbf{0.750} & 0.746 & 0.736 & 0.749 & 0.685 & \textbf{0.750} \\
Fitted Method & RF & RF & RF & RF & RF & RF \\
\hline
\end{tabular}
\label{tab:feature_selection}
\end{table}

We report results in terms of in-sample and out-of-sample recall, precision and precision-recall area under curve (PR-AUC) which are standard accepted performance metrics \citep{chawla2002smote}. The evaluation is conducted on the full feature set, features selected by benchmark models, and two variants of MINERVA with regularization coefficients set to $10^3$ and $10^4$. When optimizing FLAML, Random Forest (RF) was found to be the best model in all cases. In this experiment, we evaluated our method against the three highest performing baselines. Besides, we optimized the benchmarks to determine the number of informative features without constraining the total feature set. First, we observe that this dataset presents a significant challenge, as none of the feature selection methods yield substantial improvements over using all features. When the regularization coefficient was set to $10^3$, MINERVA selected 160 features and achieved the highest out-of-sample recall of 0.573, demonstrating strong performance on real-world datasets. Increasing the regularization coefficient to $10^4$ degraded the model's performance, indicating its sensitivity to the regularization parameter. While KSG and HSCI Lasso selected the highest number of features, the results indicate that KSG achieved slightly better performance than HSCI Lasso with regards to out-of-sample precision. KSG, HSCI Lasso, and MINERVA demonstrated strong performance based on out-of-sample PR-AUC, and Boruta consistently underperformed across all metrics. This is not surprising given that Boruta selected the smallest subset of features.

\section{Conclusions}
\label{sec.conclusion}
We presented MINERVA, a feature selection method based on neural estimation of mutual information. We validated our approach using synthetic and real-world datasets. Synthetic data was generated to address a prevalent dependence mechanism that is rarely captured by existing methods. The target variable was derived from a transformation of specific continuous features, where the transformation method depends on whether two discrete variables are equal or not.  Results on synthetic data showed a substantial improvement of our method over existing baselines, with MINERVA being the only method that selected the exact features.  

We also evaluated MINERVA on a real-life, highly unbalanced dataset (Card-fraud), where the minority class accounts for only 0.1\% of the 3 million observations. SMOTE was employed to over-sample the minority class and we performed experiments to determine a subset of features that are most relevant for fraud prediction. Experimental results showed that our method demonstrates strong performance on real-world data. For the future work, the performance of our method on more real-world applications such as bioinformatics, computer vision, and speech and signal processing will be investigated.

\section{Acknowledgments}
This work was supported by the Innovate UK Business Connect.

\appendix
\section{Proof of \textbf{Lemma 1}}
\label{app1}
\begin{proof}[Proof of Lemma \ref{lemma.experiment1}]
	For ease of notation, 
	take 
	$k_0 = 1$, 
	$k_1 = 2$. 
	We only need to prove $I(X_i; Y) = 0,$ 
	for $i = k_0, k_1$. 
	For integers $i$, $y$, 
	let 
	\begin{equation*}
		a(y, i) = \one(y = i) 
		= 
		\begin{cases}
			1 & \text{ if } y = i
			\\
			0 & \text{ otherwise}.
		\end{cases}
	\end{equation*}
	For 
	$x_1 = 1, \dots, m$
	and
	$y = 0, 1$
	we have
	\begin{equation*}
		\Prob(Y = y \vert X_1 = x_1)
		=
		\begin{cases}
			\Prob(X_2 \neq x_1) & \text { if } y = 0
			\\
			\Prob(X_2 = x_1) & \text { if } y = 1
		\end{cases}
		=
		\frac{m-1}{m} a(y, 0)
		+
		\frac{1}{m} a(y, 1)
	\end{equation*}
	Therefore,
	\begin{equation*}
		\begin{split}
			\Prob(Y = y) 
			& =
			\sum_{x_1 = 1}^{m}
			\Prob(X_1 = x_1, Y = y)
			\\
			& = 
			\sum_{x_1 = 1}^{m}
			\Prob(Y = y \vert X_1 = x_1)
			\Prob(X_1 = x_1)
			\\
			& = 
			\frac{1}{m}
			\sum_{x_1 = 1}^{m}
			\left(
			\frac{m-1}{m} a(y, 0)
			+
			\frac{1}{m} a(y, 1)
			\right)
			\\
			& = 
			\Prob(Y = y \vert X_1 = x_1)
		\end{split}
	\end{equation*}
     \vspace{-.5cm}
	where on the last line $x_1$ is any positive integer smaller than or equal to $m$. 
	We conclude that
	\begin{equation*}
		\begin{split}
		I(X_1; Y)
		& = 
		\sum_{x_1 = 1}^{m}
		\sum_{y=0}^{1}
		\Prob(X_1 = x_1, Y = y)
		\log
		\left(
		\frac{
		\Prob(X_1 = x_1, Y = y)
		}
		{
			\Prob(X_1 = x_1)\Prob(Y=y)
		}
		\right)
		\\
		& = 
		\sum_{x_1 = 1}^{m}
		\sum_{y=0}^{1}
		\Prob(X_1 = x_1, Y = y)
		\log
		\left(
		\underbrace{
		\frac{
		\Prob(X_1 = x_1, Y = y)
		}
		{
			\Prob(X_1 = x_1)
			\Prob(Y = y \vert X_1 = x_1)
		}
		}_{ = 1}
		\right)
		\\
		& = 
		0.
		\end{split}
	\end{equation*}
	The equality $I(X_2; Y) = 0$ is proved in the same way. 

	Finally, we establish equation \eqref{eq.exp1mi}.
	For integers $x_1, x_2$, 
	let 
	$b(x_1, x_2) = 1$ if $x_1 = x_2$,
	and $b(x_1, x_2) = 0$ otherwise.
	Then, 
	for positive integers $x_1, x_2 \leq m$ and $y=0, 1$,
	we can write
	\begin{equation*}
		\Prob(Y=y \vert X_1 = x_1, X_2 = x_2)
		=
		a(y, 0)( 1 - b(x_1, x_2) )
		+
		a(y, 1) b(x_1, x_2),
	\end{equation*}
	and
	\begin{equation*}
		\begin{split}
		\Prob(X_1 = x_1, X_2 = x_2, Y = y)
			&=
		\Prob(Y=y \vert X_1 = x_1, X_2 = x_2)
		\Prob(X_1 = x_1, X_2 = x_2)
		\\
			&=
		\frac{1}{m\squared}
		\Big(
		a(y, 0)( 1 - b(x_1, x_2))
		+
		a(y, 1) b(x_1, x_2)
		\Big),
		\end{split}
	\end{equation*}
	and
	\begin{equation*}
		\begin{split}
			\Prob(X_1 = x_1, X_2 = x_2) \Prob(Y = y)
			&=
			\frac{1}{m\squared}
			\left(
			\frac{m-1}{m} a(y, 0)
			+
			\frac{1}{m} a(y, 1)
			\right).
		\end{split}
	\end{equation*}
	Let 
	$c(x_1, x_2, y) = 
			a(y, 0)( 1 - b(x_1, x_2) )
			+
			a(y, 1) b(x_1, x_2)
			$.
	Plugging these in the definition of the mutual information between $(X_1, X_2)$ and $Y$,
	we conclude
	\begin{align*}
			I(X_1, X_2; Y)
			& =
			\frac{1}{m\squared}
			\sum_{x_1, x_2 = 1}^{m}
			\sum_{y=0}^{1}
			\left(
			c(x_1, x_2, y)
			\right)
			\log
			\left(
			\frac
			{
			c(x_1, x_2, y)
			}
			{
				\frac{m-1}{m} a(y, 0) + \frac{1}{m} a(y, 1)
			}
			\right)
			\\
			& = 
			\frac{1}{m\squared}
			\sum_{x_1, x_2 = 1}^{m}
			\left(
			(1 - b(x_1, x_2))
			\log
			\left(
			\frac{m(1-b(x_1, x_2))}{m-1}
			\right)
			+
			b(x_1, x_2)
			\log
			\left(
			m b(x_1, x_2)
			\right)
			\right)
			\\
			& =
			\frac{1}{m\squared}
			\sum_{x_1 = 1}^{m}
			\left(
			(m-1)\log\left(\frac{m}{m-1}\right) + \log(m)
			\right)
			\\
			& =
			\frac{m-1}{m} \log\left(\frac{m}{m-1}\right)
			+
			\frac{1}{m}\log m
			.
	\end{align*}
\end{proof}

\bibliographystyle{elsarticle-num} 
\bibliography{bibliography}

\end{document}